\RequirePackage{amsmath}
\RequirePackage{mathptmx} 
\documentclass[runningheads]{llncs}
\usepackage{graphicx}

\usepackage{amssymb,amsthm}
\usepackage{mathptmx}
\usepackage{mathtools}
 \usepackage{enumerate}
 \usepackage{cite}
\DeclareMathOperator*{\argmin}{arg\,min}
\def\one{\mbox{1\hspace{-4.25pt}\fontsize{12}{14.4}\selectfont\textrm{1}}}
\usepackage{arydshln}

\begin{document}
\title{Deceptive Reinforcement Learning Under Adversarial Manipulations on Cost Signals}
%
%
\author{Yunhan Huang\inst{1}\orcidID{0000-0002-4395-0642} \and
Quanyan Zhu\inst{1}\orcidID{0000-0002-0008-2953}}
\authorrunning{Y. Huang and Q. Zhu.}
%
\institute{New York University, New York City, NY 10003, USA \\
\email{\{yh.huang,qz494\}@nyu.edu}}
\maketitle  
\begin{abstract}
This paper studies reinforcement learning (RL) under malicious falsification on cost signals and introduces a quantitative framework of attack models to understand the vulnerabilities of RL. Focusing on $Q$-learning, we show that $Q$-learning algorithms converge under stealthy attacks and bounded falsifications on cost signals. We characterize the relation between the falsified cost and the $Q$-factors as well as the policy learned by the learning agent which provides fundamental limits for feasible offensive and defensive moves. We propose a robust region in terms of the cost within which the adversary can never achieve the targeted policy. We provide conditions on the falsified cost which can mislead the agent to learn an adversary's favored policy. A numerical case study of water reservoir control is provided to show the potential hazards of RL in learning-based control systems and corroborate the results. 

\end{abstract}
\section{Introduction}

Reinforcement Learning (RL) is a paradigm for making online decisions in uncertain environment. Recent applications of RL algorithms to Cyber-Physical Systems enables real-time data-driven control of autonomous systems and improves the system resilience to failures.  However, the integration of RL mechanisms also exposes CPS to new vulnerabilities. One type of threats arises from the feedback architecture of the RL algorithms depicted in Fig. \ref{fig:SystemLevelAttacks}. An adversary can launch a man-in-the-middle attack to delay, obscure and manipulate the observation data that are needed for making online decisions. This type of adversarial behavior poses a great threat to CPS. For example, self-driving platooning vehicles can collide with each other when their observation data are manipulated \cite{behzadan2018adversarial}. Similarly, drones can be weaponized by terrorists to create chaotic and vicious situations where they are commanded to collide to a crowd or a building.

Hence it is imperative to understand the adversarial behaviors of RL and establish
a theoretic framework to analyze the impact of the attacks on RLs. One key aspect that makes RL security unique is its feedback architecture which includes components of sensing, control, and actuation as is shown in  Fig. \ref{fig:SystemLevelAttacks}. These components are subject to different types of cyber threats. For example, during the learning process, agent learns optimal policy from sequential observations from the environment. An adversary may perturb the environment to deteriorate the learning results. This type of attack is called environment attack. Agents observe the environment via their sensors. But the sensory observation of the state may be delayed, perturbed, or falsified under malicious attacks which are usually called sensors attack. There are also actuator attacks and attacks on reward/cost signals. The latter refers to manipulation of the reward signal produced by the environment in response to the actions applied by a RL agent, which can significantly affect the learning process. Take a RL-based Unmanned Aerial Vehicle (UAV) as an example, if the reward depends on the distance of the UAV to a desired destination measured by GPS coordinates, spoofing of GPS signals by the adversary may result in incorrect reward/cost signals. 



\begin{figure}\centering
\includegraphics[width=0.7\textwidth]{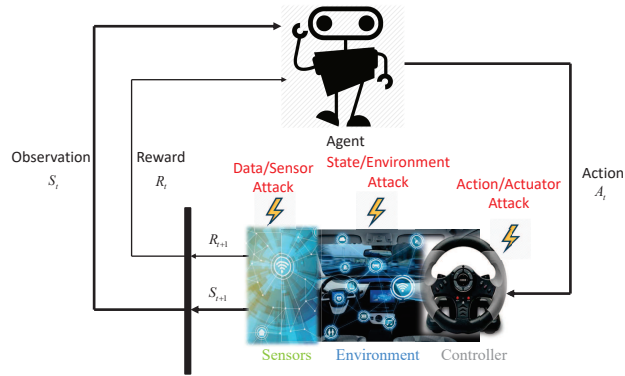}
\caption{Main components of a RL agent and potential attacks that can be applied to these components.} \label{fig:SystemLevelAttacks}
\end{figure}

In this paper, we study RL under malicious manipulation of cost signals from an offensive perspective where an adversary/attacker maliciously falsifies the cost signals. We first introduce a general formulation of attack models by defining the objectives, information structure and the capability of an adversary. We focus our research on a class of $Q$-learning algorithm and aim to address  two fundamental questions. The first one is on the impact of the falsification of cost signals on the convergence of $Q$-learning algorithm. The second one is on how the RL algorithm can be misled under malicious falsifications. We show that under stealthy attacks and bounded falsifications on the cost signals, the $Q$-learning algorithm converges almost surely. If the algorithm converges, we characterize the relationship between the falsified cost and the limit of $Q$-factors by an implicit map. We show that the implicit map has several useful properties including differentiability, Lipschitz continuity etc, which help to find fundamental limits of adversarial behavior. In particular, from the implicit map, we study how the falsified cost affect the policy that agents learn. We show that the map is uniformly Lipschitz continuous with an explicit Lipschitz constant and based on this, we characterize a robust region where the adversary can never achieve his desired policy if the falsified cost stays in the robust region. The map is shown to be Fr\'echet differentiable almost everywhere and Fr\'echet derivative is explicitly characterized which is independent of the falsified cost. The map has `piece-wise linear' property on a normed vector space. The derivative and `piece-wise linear' property can be utilized by the adversary to drive the $Q$-factors to a desired region by falsifying cost signals properly. We show that once the falsified cost satisfies a set of ineqialities, the RL agent can be mislead to learn the policy manipulated by the adversary. Further, we give conditions under which the adversary can attain any policy even if the adversary is only capable of falsifying the cost at a subset of the state space. In the end, An example is presented to illustrate potential hazards that might be caused by malicious cost falsification. The main contributions of our paper can be summarized as follows:
\begin{enumerate}
    \item We establish a theoretic framework to study strategic manipulation/falsifications on cost signals in RL and present a set of attack models on RL. 
    \item We provide an analytical results to understand how falsification on cost singals can affect $Q$-factors and hence the policies learned by RL agents. 
    \item We characterize conditions on deceptively falsified cost signals under which $Q$-factors learned by agents can produce the policy that adversaries aim for.
    \item We use a case study of water reservoir to illustrate the severe damages of insecure RL that can be inflicted on critical infrastructures and demonstrate the need for defense mechanisms for RL. 
\end{enumerate}

\subsection{Related Works}
Very few works have explicitly studied security issues of RL \cite{behzadan2018faults}. There is a large literature on adversarial machine learning, whose focus is on studying the vulnerability of supervised learning. However, we aim to provide a fundamental understanding of security risks of RL which is different from both supervised learning and unsupervised learning \cite{sutton1998introduction}. So, there remains a need for a solid theoretic foundation on security problems of RL so that many critical applications would be safeguarded from potential RL risks.

One area relevant to security of RL is safe RL \cite{garcia2015comprehensive}, which aims to ensure that agents learn to behave in compliance with some pre-defined criteria. The security problem, however, is concerned with settings where an adversary intentionally seeks to compromise the normal operation of the system for malicious purposes \cite{behzadan2018faults}. Apart from the distinction between RL security and safe RL, the difference between RL security and the area of adversarial RL also exists. The adversarial RL is usually studied under multi-agent RL settings, in which agents aim to maximize their returns or minimize their cost in competition with other agents. 

There are two recent works that have studied inaccurate cost signals. In \cite{everitt2017reinforcement}, Everitt et. al. study RL for Markov Decision Process with corrupted reward channels where due to some sensory errors and software bugs, agents may get corrupted reward at certain states. But their focus is not on security perspectives and they look into unintentional perturbation of cost signals. In \cite{wang2018reinforcement}, Wang et. al. have studied $Q$-learning with perturbed rewards where the rewards received by RL agents are perturbed with certain probability and the rewards take values only on a finite set. They study unintentional cost perturbation from a robust perspective other than a security perspective. Compared the two works mentioned above, our work studies RL with falsified cost signals from a security point of view and we develop theoretical underpinnings to characterize how the falsified cost will deteriorate the learning result.

The falsification of cost/reward signals can be viewed as one type of deception mechanisms. The topic of defensive deception has bee surveyed in \cite{pawlick2017game}, which includes a taxonomy of deception mechanisms and a review of game-theoretic models. Game and decision-theoretic models for deception have been studied in various contexts \cite{zhang2019game,horak2017manipulating}, including honeypots \cite{pawlick2018modeling,jeffWiopt}, adversarial machine learning \cite{zhang2018gameML,zhang2015secure}, moving target defense \cite{zhu2013game,clark2012deceptive}, and cyber-physical control systems \cite{zhu2015game,pawlick2018istrict,pawlick2017strategic,rass2017physical}. In this work, we extend the paradigm of cyber deception to reinforcement learning and establish a theoretical foundation for understanding the impact and the fundamental limits of such adversarial behaviors.

\subsection{Organization of the Paper}
In Section \ref{PrePro}, we present preliminaries and formulate a general framework that studies several attack models. In Section \ref{AnaQ}, we analyze the $Q$-learning algorithm under adversarial manipulations on cost. We study under what conditions the $Q$-learning algorithm converges and where it converges to. In Section \ref{ExamNum}, we present an example to corroborate the theoretical results and their implications in the security problems of RL.


\section{Preliminaries and Problem Formulation}\label{PrePro}
\subsection{Preliminaries}
Consider one RL agent interacts with an unknown environment and attempts to minimize the total of its received costs. The environment is formalized as a Markov Decision Process (MDP) denoted by $\langle \mathcal{S}, \mathcal{A}, c, \mathcal{P},\beta \rangle$. The MDP $\{\Phi(t):t\in\mathbb{Z}\}$ takes values in a finite state space $\mathcal{S}=\{1,2,...,S\}$ and is controlled by a sequence of actions (sometimes called a control sequence) $\mathbf{Z}=\{Z(t):t\in\mathbb{Z}\}$ taking values in a finite action space $\mathcal{A}=\{a_1,...,a_A\}$. Throughout this paper, we use the term action sequence and control sequence interchangeably. In our setting, we are interested in stationary policies where the control sequence takes the form $Z(t)=w(\Phi(t))$, where the feedback rule $w$ is a function $w:\mathcal{S}\rightarrow \mathcal{A}$. To emphasize the policy $w$, we denote $\mathbf{Z}_w=\{Z_w(t)\coloneqq w(\Phi(t)):t\in\mathbb{Z}\}$. According to a transition probability kernel $\mathcal{P}$, the controlled transition probabilities are given by $p(i,j,a)$ for $i,j\in \mathcal{S},a\in \mathcal{A}$. Commonly $\mathcal{P}$ is unknown to the agent.

Let $c: \mathcal{S}\times \mathcal{A} \rightarrow \mathbb{R} $ be the one-step cost function, and consider the infinite horizon discounted cost control problem of minimizing over all admissible $\mathbf{Z}$ the total discounted cost $
J(i,\mathbf{Z})=\mathbf{E}[ \sum_{t=0}^\infty \beta^t c(\Phi(t),Z(t))|\Phi(0)=i ],$ where $\beta\in(0,1)$ is the discount factor. The minimal value function is defined as $V(i)=\min J(i,\mathbf{Z})$, where the minimum is taken over all admissible control sequences $\mathbf{Z}$. The function $V$ satisfies the dynamic programming equation \cite{bertsekas1996neuro}
$$
V(i)=\min_a \Big[ c(i,a)+\beta \sum_j p(i,j,a)V(j) \Big],\ \ i\in \mathcal{S}
$$
and the optimal control minimizing $J$ is given by the stationary policy defined through the feedback law $w^*$ given by
$w^*(i)\coloneqq \argmin_a [ c(i,a) + \beta \sum_j p(i,j,a) V(j)]$, $i\in \mathcal{S}.$ If we define $Q$-values via
$$
Q(i,a)=c(i,a)+\beta \sum_j p(i,j,a)V(j), i\in \mathcal{S}, a\in \mathcal{A},
$$
then $V(i)=\min_a Q(i,a)$and the matrix $Q$ satisfies
\begin{equation}\label{QDPEquation}
   Q(i,a)=c(i,a) + \beta \sum_j p(i,j,a)\min_b Q(j,b),\ \ i\in \mathcal{S}, a\in \mathcal{A}. 
\end{equation}
If the matrix $Q$ defined in (\ref{QDPEquation}) can be computed, e.g., using value iteration, then the optimal control policy can be found by $w^*(i)=\argmin_a Q(i,a),i\in\mathcal{S}$. When transition probabilities are unknown, we can use a variant of stochastic approximation known as the $Q$-learning algorithm proposed in \cite{watkins1992q}. The learning process is defined through the recursion
\begin{equation}\label{QLAl}
\begin{aligned}
&Q_{n+1}(i,a) = Q_n(i,a) + a(n) \times
\Big[ \beta \min_b Q_n(\Psi_{n+1}(i,a),b) + c(i,a) - Q_n(i,a) \Big],
\end{aligned}
\end{equation}
$i\in \mathcal{S}, a\in \mathcal{A},$ where $\Psi_{n+1}(i,a)$ is an independently simulated $\mathcal{S}$-valued random variable with law $p(i,\cdot,a)$.

\subsubsection{Notations}
An indicator function $\one_C$ is defined as $\one_C(x)=1$ if $x\in C$, and $\one_C(x)=0$ otherwise. Denote $\mathbf{1}_i\in\mathbb{R}^S$ a vector with $S$ components whose $i$th component is $1$ and other components are $0$. The true cost at time $t$ is denoted by the shorthand notion $c_t\coloneqq c(\Phi(t),Z(t))$. For a mapping $f:\mathbb{R}^{S\times A}\rightarrow \mathbb{R}^{S\times A}$, define $f_{ia}:\mathbb{R}^{S\times A}\rightarrow\mathbb{R}$ that maps $\mathbb{R}^{S\times A}$ to $\mathbb{R}$ where for any $Q\in\mathbb{R}^{S\times A}$, we have $[f(Q)]_{i,a}=f_{ia}(Q)$ and $[f(Q)]_{i,a}$ is the $i$th component and $a$th column of $f(Q)$. The inverse of $f$ is denoted by $f^{-1}$. Given a set $\mathcal{V} \subset \mathbb{R}^{S\times A}$, $f^{-1}(\mathcal{V})$ is referred to the set $\{c: f(c)\in\mathcal{V} \}$. Denote $\mathcal{B}(c;r)\coloneqq\{\tilde{c}:\Vert \tilde{c}-c\Vert< r  \}$ an open ball in a normed vector space with radius $r$ and center $c$. Here and in later discussion, $\Vert \cdot \Vert$ refers to the maximum norm. 

Given $c\in\mathbb{R}^{S\times A}$ and a policy $w$, denote $c_w\in\mathbb{R}^{S}$ a vector whose $i$th component is $c(i,w(i))$ for any $i\in\mathcal{S}$. Define $c_a\in\mathbb{R}^S$ as a vector whose $i$th component is $c(i,a)$. We define $Q_{w}$, $Q_a$ in the same way. For transition probability, we define $P_w\in\mathbb{R}^{S\times S}$ as $[P_{w}]_{i,j}=p(i,j,w(i))$ and $P_{ia}=(p(i,1,a),p(i,2,a),...,p(i,S,a))^T\in\mathbb{R}^S$. Define $P_a\in\mathbb{R}^{S\times S}$ as the matrix whose components are $[P_a]_{i,j}=p(i,j,a)$. 

\subsection{General Attack Models}\label{GenAttModel}
Under malicious attacks,  the RL agent will not be able to observe the true cost feedback from the environment. Instead, the agent is given a cost signal that might be falsified by the attacker. Consider the following MDP with falsified cost (MDP-FC) denoting as $\langle \mathcal{S}, \mathcal{A}, c, \tilde{c},\mathcal{P},\beta \rangle$. In MDP-FC, at each time $t$, instead of observing $c_t\in\mathbb{R}$ directly, the agent only observes a falsified cost signal denoted by $\tilde{c}_t\in \mathbb{R}$. The remaining aspects of the MDP framework stay the same. 

Attack models can be specified by three components: objective of an adversary, actions available to the adversary, and information at his disposal. The adversary's task here is to design falsified cost signals $\tilde{c}$ based on his information structure and the actions available to him so that he can achieve certain objectives.

\textit{Objective of Adversary:}  One possible objective of an adversary is to maximize the agent's cost while minimizing the cost of attacks. This type of objectives can be can captured by a cost function 
$$
\max_{\tilde{c}}\ \ \mathbf{E}\Big[\sum_{t=0}^\infty \beta^t c(\Phi(t),Z_{w(\tilde{c})}(t)) \Big]- \textrm{AttackCost}(\tilde{c}).
$$
The other adversarial objectives would be to drive the MDP to a targeted process or to mislead the agent to learn certain policies the attacker aims for. Let $w(\tilde{c})$ denote the policy learned by the agent under falsified cost signals $\tilde{c}$ and let $w^\dagger$ denote the policy that an attacker aims for. We can capture the objective of such a deceptive adversary by 
\begin{equation}\label{AttackPolicyObjctive}
\max\limits_{\tilde{c}}\ \ \one_{\{w^\dagger\} }( w(\tilde{c}) ) - \textrm{AttackCost}(\tilde{c}).
\end{equation}

Here, the second term $\textrm{AttackCost}(\tilde{c})$  serves as a measure for the cost of attacking while the first term indicates whether the agent learns the policy $w^\dag$ or not. We can, for example, define $ \textrm{AttackCost}(\tilde{c})= \sum_{t=0}^\infty \alpha^t d(c_t,\tilde{c}_t),$
where $d(\cdot,\cdot)$ is a metric, $\alpha$ is a discount factor. If $d$ is a discrete metric, then $\sum_{t=0}^T d(c_t,\tilde{c}_t)$ counts the number of times of cost signals being falsified before time $T$.
Note that here, $\tilde{c}$ represents all the possible ways that an adversary can take to generate falsified signals.

\textit{Information:} It is important to specify the information structure of an adversary which determines different classes of the attacks an adversary can launch. We can categorize them as follows.

\begin{definition}
\begin{enumerate}
    \item An attacker is called an omniscient attacker if the information the attacker has at time $t$, denoted by, $\mathcal{I}_t$, is defined as 
    $$
    \mathcal{I}_t=\{\mathcal{P},\Phi(\tau),Z(\tau),c:\tau\leq t\}.
    $$
    \item An attacker is called a peer attacker if the attacker has only access to the knowledge of what the agent knows at time $t$. That means
    $$
    \mathcal{I}_t=\{\Phi(\tau),Z(\tau),c_\tau:\tau\leq t \}
    $$
    \item An attacker is called an ignorant attacker if at time $t$, he only knows the cost signals before time $t$, i.e.,
    $$
    \mathcal{I}_t=\{c_\tau:\tau\leq t\}
    $$
    \item An attacker is called a blind attacker if the information the attacker has at time $t$, denoted by $\mathcal{I}_t$, is defined as 
    $$
    \mathcal{I}_t=\varnothing.
    $$
\end{enumerate}
\end{definition}

\begin{remark}
There are many other situations in terms of information sets of an attacker that we can consider. In the definition of an omniscient attacker,  $c$ represents the true cost at every state-action pair. One should differentiate it from $c_\tau$. The latter means the true cost generated at time $\tau$. That is to say an omniscient attack knows the true cost at every state action pair $(i,a)$ for all $t$.
\end{remark}

\textit{Actions Available:} Even if an adversary can be omniscient, it does not mean that he can be omnipotent. The actions available to an adversary need to be defined. For example, the attacker can only create bounded perturbations to true cost signals. In some cases, the action of an adversary may be limited to changing the sign of the cost at certain time or he can only falsify the cost signals at certain states in the subset $\mathcal{S}'\subset \mathcal{S}$.

The constraints on the actions available to an attacker can also be captured by the attack cost. The cost for the type of attacks whose actions are constrained to a subset $\mathcal{S}'$ can be captured by the following
    $$
    \begin{aligned}
     \textrm{AttackCost}(\tilde{c}) =\begin{cases} 0\ \ \ \ \textrm{if\ }\tilde{c}_t = c_t \coloneqq c(\Phi(t),Z(t)),\textrm{for }\Phi(t)\in S\backslash \tilde{S},\forall t\\
     \infty\ \ \textrm{Otherwise}.
     \end{cases}
     \end{aligned}
    $$

Moreover, the generation of falsified costs relies heavily on the information an attacker has. If the attacker is a peer attacker or an omniscient attacker, the falsified signal $\tilde{c}$ can be generated through a mapping $C:\mathcal{S}\times \mathcal{A} \times \mathbb{R}\rightarrow \mathbb{R}$, i.e., $\tilde{c}_t=C(\Phi(t),Z(t),c_t)$. If the attacker only knows the state and the cost,  $\tilde{c}$ can be generated by the mapping $C:\mathcal{S}\times \mathbb{R}\rightarrow \mathbb{R}$. If the attacker is ignorant, we have $C:\mathbb{R}\rightarrow \mathbb{R}$, then $\tilde{c}_t=C(c_t)$.

\begin{definition}[Stealthy Attacks]\label{StealthyAttack}
If $\tilde{c}_t$ takes the same value for the same state-action pair $(\Phi(t),Z(t))$ for all $t\in\mathbb{Z}$, i.e., for $t\neq\tau$, if for $(\Phi(t),Z(t))=(\Phi(\tau),Z(\tau))$, we have $\tilde{c}_t=\tilde{c}_\tau$, then we say that the attacks on the cost signals are stealthy.
\end{definition}

The definition states that the cost falsification remains consistent for the same state-action pairs. In later discussions, we focus on stealthy attacks, which is a class of attacks that are hard to detect. 
Under stealthy attackers, the falsified cost $\tilde{c}$ can be viewed as a falsified cost matrix of dimension $S\times A$. At time $t$, the cost received by the RL agent is $\tilde{c}(\Phi(t),Z(t))$.

\subsection{$Q$-Learning with Falsified Cost}

If the RL agent learns an optimal policy by $Q$-learning algorithm given in (\ref{QLAl}), then under stealthy attacks on cost, the algorithm can be written as 
\begin{equation}\label{FalQLe}
\begin{aligned}
Q_{n+1}(i,a) = Q_n(i,a) + a(n) \times \Big[ \beta \min_b Q_n(\Psi_{n+1}(i,a),b) + \tilde{c}(i,a) - Q_n(i,a) \Big].
\end{aligned}
\end{equation}
Note that if the attacks are not stealthy, we need to write $\tilde{c}_n$ in lieu of $\tilde{c}(i,a)$. There are two important questions regarding the $Q$-learning algorithm with falsified cost (\ref{FalQLe}): (1) Will the sequence of $Q_n$-factors converge? (2) Where will the sequence of $Q_n$ converge to? We will address these two issues in next section.

Suppose that the sequence $Q_n$ generated by the $Q$-learning algorithm (\ref{FalQLe}) converges. Let $\tilde{Q}^*$ be the limit, i.e., $\tilde{Q}^* = \lim_{n\rightarrow\infty} Q_n$. Suppose the objective of an adversary is to induce the RL agent to learn a particular policy $w^\dagger$. The adversary's problem then is to design $\tilde{c}$ by applying the actions available to him based on the information he has so that the limit $Q$-factors learned from the $Q$-learning algorithm produce the policy targeted by the adversary $w^\dagger$, i.e, $\tilde{Q}^*\in \mathcal{V}_{w^\dagger} $, where 
$$
\mathcal{V}_{w}\coloneqq \{Q\in\mathbb{R}^{S\times A}: w(i)=\argmin_a Q(i,a), \forall i\in \mathcal{S} \}.
$$

In next section, we will develop theoretical underpinnings to address the issues regarding the convergence of (\ref{FalQLe}) and the attainability of the adversarial objectives. 

\section{Analysis of $Q$-Learning with Falsified Cost}\label{AnaQ}

\subsection{Convergence of $Q$-Learning Algorithm with Falsified Cost}

In $Q$-learning algorithm (\ref{QLAl}), to guarantee almost sure convergence, the agent usually takes tapering stepsize \cite{borkar2009stochastic} $\{a(n)\}$ which satisfies $0<a(n)\leq 1$, $n\geq 0$, and $\sum_n a(n) =\infty$, $\sum_n a(n)^2 < \infty$. Suppose in our problem, the agent takes tapering stepsize. To address the convergence issues, we have the following result.
\begin{lemma}\label{ConvergenceFalQ}
If an adversary performs stealthy attacks with bounded $\tilde{c}(i,a)$ for all $i\in\mathcal{S},a\in\mathcal{A}$, then the $Q$-learning algorithm with falsified costs converges to the fixed point of $\tilde{F}(Q)$ almost surely where the mapping $\tilde{F}:\mathbb{R}^{S\times A} \rightarrow \mathbb{R}^{S\times A}$ is defined as $\tilde{F}(Q)=[\tilde{F}_{ia}(Q)]_{i,a}$ with
$$
\tilde{F}_{ia}(Q)=\beta \sum_j p(i,j,a)\min_b Q(j,b) +\tilde{c}(i,a),
$$
and the fixed point is unique and denoted by $\tilde{Q}^*$.
\end{lemma}

\begin{proof}[Sketch of Proof]
 If the adversary performs stealthy attacks, the falsified costs for each state-action pair are consistent during the learning process. The $Q$ learning process thus can be written as (\ref{FalQLe}). Rewrite (\ref{FalQLe}) as $Q_{n+1}=Q_{n}+a(n)\big[ \tilde{h}(Q_n)+M(n+1) \big]$,
where $\tilde{h}(Q)\coloneqq\tilde{F}(Q)-Q$ and $M(n+1)$ is given as
$$
M_{ia}(n+1)
=\beta \Bigg( \min_b Q_n(\Psi_{n+1}(i,a),b)
-\sum_j p(i,j,a)(\min_b Q_n(j,b)) \Bigg),\ i\in S,a\in A.
$$
Note that for any $Q_1,Q_2$, $\tilde{h}(Q_1)-\tilde{h}(Q_2)$ and $\tilde{F}(Q_1)-\tilde{F}(Q_2)$ do not depend on the falsified costs. If the falsified costs are bounded, one can see that $\tilde{h}(Q)$ is Lipschitz. And $M(n+1)$ is a Martingale difference sequence. Following the arguments in \cite{borkar2009stochastic} (Theorem 2 Chapter 2) and Section 3.2 of \cite{borkar2000ode}, we can see the iterates of (\ref{FalQLe}) converges almost surely to the fixed points of $\tilde{F}$. Since $\tilde{F}$ is a contraction mapping with respect to the max norm, with contraction factor $\beta$ \cite{bertsekas1996neuro} (pp. 250), by Banach fixed point theorem (contraction theorem), $\tilde{F}$ admits a unique fixed point.
\end{proof}

It is not surprising that one of the conditions given in Lemma \ref{ConvergenceFalQ} that guarantees convergence is that an attacker performs stealthy attacks. The convergence can be guaranteed because the falsified cost signals are consistent over time for each state action pair. The uniqueness of $\tilde{Q}^*$ comes from the fact that if $\tilde{c}(i,a)$ is bounded for every $(i,a)\in \mathcal{S}\times\mathcal{A}$, $\tilde{F}$ is a contraction mapping. By Banach's fixed point theorem \cite{kreyszig1978introductory}, $\tilde{F}$ admits a unique fixed point. With this lemma, we conclude that an adversary can make the algorithm converge to a limit point by stealthily falsifying the cost signals.

\begin{remark}
Whether an adversary aims for the convergence of the $Q$-learning algorithm (\ref{FalQLe}) or not depends on his objective. In our setting, the adversary intends to mislead the RL agent to learn  policy $w^\dag$, indicating that the adversary promotes convergence and aim to have the limit point $\tilde{Q}^*$ lie in $\mathcal{V}_{w^\dag}$.
\end{remark}

\subsection{How is the Limit Point Affected by the Falsified Cost}

Now it remains to analyze, from the adversary's perspective, how to falsify the cost signals so that the limit point that algorithm (\ref{FalQLe}) converges to is desired by the adversary. In later discussions, we consider stealthy attacks where the falsified costs are consistent for the same state action pairs. Denote the true cost by matrix $c\in\mathbb{R}^{S\times A}$ with $[c]_{i,a}=c(i,a)$ and the falsified cost is described by a matrix $\tilde{c}\in\mathbb{R}^{S\times A}$ with $[\tilde{c}]_{i,a}=\tilde{c}(i,a)$. Given $\tilde{c}$, the fixed point of $\tilde{F}$ is uniquely decided, i.e., the point that the algorithm (\ref{FalQLe}) converges to is uniquely determined. Thus, there is a mapping $\tilde{c}\mapsto\tilde{Q}^*$ implicitly described by the relation $\tilde{F}(Q)=Q$. For convenience, this mapping is denoted by $f:\mathbb{R}^{S\times A}\rightarrow \mathbb{R}^{S\times A}$.

\begin{theorem}\label{LipContTheo}
Let $\tilde{Q}^*$ denote the $Q$-factor learned from algorithm (\ref{FalQLe}) with falsified cost signals and $Q^*$ be the $Q$-factor learned from (\ref{QLAl}) with true cost signals. There exists a constant $L<1$ such that
\begin{equation}\label{LipCont}
    \Vert \tilde{Q}^* - Q^* \Vert  \leq  \frac{1}{1-L}\Vert \tilde{c}- c \Vert,
\end{equation}
and $L=\beta$ where discounted factor $\beta$ has been defined in the MDP-FC problem.
\end{theorem}
\begin{proof}
Define $\tilde{F}(Q)$ as $\tilde{F}_{ia}(Q)=\beta \sum_j p(i,j,a)\min_b Q(j,b) +c(i,a).$ From Lemma \ref{ConvergenceFalQ}, we know that $\tilde{Q}^*$ and $Q^*$ satisfy $\tilde{Q^*}=\tilde{F}(\tilde{Q}^*)$ and $Q^*=F(Q^*)$. We have $\tilde{Q}^*-\tilde{Q}=\tilde{F}(\tilde{Q}^*)-F(Q^*)$. Since $\tilde{F}$ and $F$ are both contraction mappings, by triangle inequality, we have $\Vert \tilde{Q}^* - Q^*\Vert\leq L \Vert \tilde{Q}^*-Q^* \Vert + \Vert \tilde{c}-c \Vert$. Thus, we have
(\ref{LipCont}). And the contraction factor $L$ for $\tilde{F}$ and $F$ is $\beta$. 
\end{proof}

\begin{remark}
In fact, taking this argument just slightly further, one can conclude that falsification on cost $c$ using a tiny perturbation does not cause significant changes in the limit point of algorithm (\ref{QLAl}), $Q^*$. This feature indicates that an adversary cannot cause a significant change in the limit $Q$-factor by just a small perturbation in the cost signals. 
This is a feature known as stability that is observed in problems that possess contraction mapping properties. Also, Theorem \ref{LipContTheo} indicates that the mapping $\tilde{c}\mapsto \tilde{Q}^*$ is continuous, and to be more specific, it is uniformly Lipchitz continuous with Lipchitz constant ${1}/{(1-\beta)}$.
\end{remark}

With Theorem \ref{LipContTheo}, we can now characterize the minimum level of falsification an adversary needs to change the policy from the true optimal policy $w^*$ to the policy $w^\dag$ that the adversary aims for. First, note that $\mathcal{V}_w\subset \mathbb{R}^{S\times A}$ and it can be also written as 
\begin{equation}\label{VictorySet}
\mathcal{V}_w=\{Q\in\mathbb{R}^{S\times A}: Q(i,w(i))< Q(i,a), \forall i\in\mathcal{S}, \forall a\neq w(i) \}.
\end{equation}
We can easily see that for any given policy $w$, $\mathcal{V}_w$ is a convex set, hence connected. This is because for any $\lambda\in [0,1]$, if $Q_1,Q_2\in\mathcal{V}_w$, $\lambda Q_1 +(1-\lambda)Q_2\in \mathcal{V}_w$. Second, for any two different policies $w_1$ and $w_2$, $\mathcal{V}_{w_1}\cap \mathcal{V}_{w_2}=\varnothing$. Define the infimum distance between the true optimal policy $w^*$ and the adversary desired policy $w^\dag$ in terms of the $Q$-values by 
$$D(w^*,w^\dag)\coloneqq\inf_{Q_1\in \mathcal{V}_{w^*}, Q_2\in\mathcal{V}_{w^\dag}} \Vert Q_1 - Q_2 \Vert,$$
which is also the definition of the distance between two sets $\mathcal{V}_{w^*}$ and $\mathcal{V}_{\omega^\dag}$. Note that for $w^* \neq w^\dag$ (otherwise, the optimal policy $w^*$ is what the adversary desire, there is no incentive for the adversary to attack), $D(w^*,w^\dag)$ is always zero according to the definition of the set (\ref{VictorySet}). This counterintuitive result states that a small change in the $Q$-value may result in any possible change of policy learned by the  agent from the $Q$-learning algorithm (\ref{FalQLe}).  Compared with Theorem \ref{LipContTheo} which is a negative result to the adversary, this result is in favor of the adversary.

Similarly, define the point $Q^*$ to set 
$\mathcal{V}_{w^\dag}$ distance by 
$
D_{Q^*}(w^\dag)\coloneqq \inf_{Q\in\mathcal{V}_{w^\dag}} \Vert Q-Q^* \Vert.
$
Thus, if $\tilde{Q}^*\in\mathcal{V}_{w^\dag}$, we have 
\begin{equation}\label{RobustRegionIne}
0= D(w^*,w^\dagger)\leq D_{Q^*}(w^\dag)\leq \Vert \tilde{Q}^*-Q^* \Vert  \leq\frac{1}{1-\beta}\Vert \tilde{c} - c \Vert,
\end{equation}
where the first inequality comes from the fact that $Q^*\in \mathcal{V}_{w^*}$ and the second inequality is due to $\tilde{Q}^*\in V_{w^\dag}$. The \textit{robust region} for the true cost $c$ to the adversary's targeted policy $w^\dag$ is given by $\mathcal{B}(c;(1-\beta)D_{Q^*}(w^\dag) )$ which is an open ball with center $c$ and radius $(1-\beta)D_{Q^*}(w^\dag)$. That means the attacks on the cost needs to be `powerful' enough to drive the falsified cost $\tilde{c}$ outside the ball $\mathcal{B}(c;(1-\beta)D_{Q^*}(w^\dag) )$ to make the RL agent learn the policy $w^\dag$. If the falsified cost $\tilde{c}$ is within the ball, the RL agent can never learn the adversary's targeted policy $w^\dag$. The ball $\mathcal{B}(c;(1-\beta)D_{Q^*}(w^\dag))$ depends only on the true cost $c$ and the adversary desired policy $w^\dag$ (Once the MDP is given, $Q^*$ is uniquely determined by $c$). Thus, we refer this ball as the robust region of the true cost $c$ to the adversarial policy $w^\dag$. As we have mentioned in Section \ref{GenAttModel}, if the actions available to the adversary only allows him to perform bounded falsification on cost signals and the bound is smaller than the radius of the robust region, then the adversary can never mislead the agent to learn policy $w^\dag$.

\begin{remark}
First, in discussions above, the adversary policy $w^\dag$ can be any possible polices and the discussion remains valid for any possible policies. Second, set $\mathcal{V}_{w}$ of $Q$-values is not just a convex set but also an open set. We thus can see that $D_{Q^*}(w^\dag)>0$ for any $w^\dag \neq w^*$ and the second inequality in (\ref{RobustRegionIne}) can be replaced by a strict inequality. Third, the agent can estimate his own robustness to falsification if he can know the adversary desired policy $w^\dag$. For an omniscient attacker or attackers who have access to true cost signals, the attacker can compute the robust region of the true cost to his desired policy $w^\dag$ to evaluate whether the objective is feasible or not. When it is not feasible, the attacker can consider changing his objectives, e.g., selecting other favored policies that have a smaller robust region.
\end{remark}

We have discussed how falsification affects the change of $Q$-factors learned by the agent in a distance sense. The problem now is to study how to falsify the true cost in a right direction so that the resulted $Q$-factors fall into the favored region of an adversary. One difficulty of analyzing this problem comes from the fact that the mapping $\tilde{c}\mapsto \tilde{Q}^*$ is not explicit known. The relation between $\tilde{c}$ and $\tilde{Q}^*$ is governed by the $Q$-learning algorithm (\ref{FalQLe}). Another difficulty is that due to the fact that both $\tilde{c}$ and $\tilde{Q}^*$ lies in the space of $\mathbb{R}^{S\times A}$, we need to resort to Fr\'echet derivative or G\^ateaux derivative \cite{cheney2013analysis} (if they exist) to characterize how a small change of $\tilde{c}$ results in a change in $\tilde{Q}^*$.

From Lemma \ref{ConvergenceFalQ} and Theorem \ref{LipContTheo}, we know that $Q$-learning algorithm converges to the unique fixed point of $\tilde{F}$ and that $f:\tilde{c}\mapsto \tilde{Q}^*$ is uniformly Lipschitz continuous. Also, it is easy to see that the inverse of $f$, denoted by $f^{-1}$, exists since given $\tilde{Q}^*$, $\tilde{c}$ is uniquely decided by the relation $\tilde{F}(Q)=Q$. Furthermore, by the relation $\tilde{F}(Q)=Q$, we know $f$ is both injective and surjective and hence a bijection which can be simply shown by arguing that given different $\tilde{c}$, the solution of $\tilde{F}(Q)=Q$ must be different. This fact informs that there is a one-to-one, onto correspondence between $\tilde{c}$ and $\tilde{Q}^*$. One should note that the mapping $f:\mathbb{R}^{S\times A}\rightarrow \mathbb{R}^{S\times A}$ is not uniformly Fr\'echet differentiable on $\mathbb{R}^{S\times A}$ due to the $\min$ operator inside the relation $\tilde{F}(Q)=Q$. However, for any policy $w$, $f$ is Fr\'echet differentiable on $f^{-1}(\mathcal{V}_w)$ which is an open set and connected due to the fact that $\mathcal{V}_w$ is open and connected and $f$ is continuous.

\begin{proposition}\label{Diffentialf}
The map $f:\mathbb{R}^{S\times A}\rightarrow \mathbb{R}^{S\times A}$ is Fr\'echet differentiable on $\mathcal{V}_w$ for any policy $w$ and the Fr\'echet derivative of $f$ at any point $\tilde{c}\in \mathcal{V}_w$, denoted by $f'(\tilde{c})$, is a linear bounded map $G:\mathbb{R}^{S\times A} \rightarrow \mathbb{R}^{S\times A}$ that does not depend on $\tilde{c}$, and 
$Gh$ is given as
\begin{equation}\label{FrechetDerivative}
[Gh]_{i,a}= \beta P_{ia}^T(I-\beta P_w)^{-1}h_w + h(i,a) 
\end{equation}
for every $i\in\mathcal{S},a\in\mathcal{A}$.
\end{proposition}
\begin{proof}
Suppose $c\in f^{-1}(\mathcal{V}_w)$ and $\tilde{c}=c+h\in f^{-1}(\mathcal{V}_w)$. By definition, $Q^*,\tilde{Q}^*\in\mathcal{V}_w$. By Lemma \ref{ConvergenceFalQ}, we have $\tilde{Q}^* = \tilde{F}(\tilde{Q}^*)$ and $Q^*=F(Q^*)$ which means
\begin{equation}\label{eq:Derivative}
\begin{aligned}
&\tilde{Q}^*(i,a)=\beta P_{ia}\tilde{Q}^*_w + \tilde{c}(i,a)=\beta P_{ia}\tilde{Q}^*_w + c(i,a) + h(i,a),\\
&Q^*(i,a)=\beta P_{ia} Q^*_w + c(i,a),\ \forall i\in\mathcal{S}, a\in\mathcal{A}.
\end{aligned}
\end{equation}
From (\ref{eq:Derivative}), we have $Q^*_w=\beta P_w Q^*_w + c_w$. Thus, $Q^*_w=(I-\beta P_w)^{-1}c_w$. Similarly, $\tilde{Q}^*_w=(I-\beta P_w)^{-1}(c_w + h_w)$, where $(I-\beta P_w)$ is invertible due to the fact that $\beta<1$ and $P_w$ is a stochastic matrix. Thus, $\tilde{Q}^*_w = Q^*_w + (I-\beta P_w)^{-1}h_w$. Substitute it into the first equation of (\ref{eq:Derivative}), one have
$$
\begin{aligned}
\tilde{Q}^*(i,a)&=\beta P_{ia}(Q^*_w + (I-\beta P_w)^{-1}h_w)+c(i,a)+h(i,a)\\
&=Q^*(i,a)+\beta P_{ia}(I-\beta P_w)^{-1}h_w+h(i,a).
\end{aligned}
$$
Then, one can see $\Vert f(c+h)-f(c)-Gh \Vert/ \Vert h\Vert \rightarrow 0$ as $\Vert h \Vert\rightarrow 0$.
\end{proof}

From Proposition \ref{Diffentialf}, we can see that $f$ is Fr\'echet differentiable on $f^{-1}(\mathcal{V}_w)$ and the derivative is constant, i.e., $f'(\tilde{c})=G$ for any $\tilde{c}\in f^{-1}(\mathcal{V}_w)$. Note that $G$ lies in the space of all linear mappings that maps $\mathbb{R}^{S\times A}$ to itself and $G$ is determined only by the discount factor $\beta$ and the transition kernel $\mathcal{P}$ of the MDP problem. The region where the differentiability may fail is  $f^{-1}(\mathbb{R}^{S\times A}\backslash (\cup_{w} \mathcal{V}_w))$, where $\mathbb{R}^{S\times A}\backslash (\cup_{w} \mathcal{V}_w)$ is the set $\{Q: \exists i, \exists a=a', Q(i,a)=Q(i,a')=\min_b Q(i,b) \}$. This set contains the places where a change of policy happens, i.e., $Q(i,a)$ and $Q(i,a')$ are both the lowest value among the $i$th row of $Q$. Also, due to the fact that $f$ is Lipschitz, by Rademacher's theorem, $f$ is differentiable almost everywhere (w.r.t. the Lebesgue measure).

\begin{remark}
One can view $f$ as a `piece-wise linear function' in the norm vector space $\mathbb{R}^{S\times A}$ instead of in a real line. Actually, if the adversary can only falsify the cost at one state-action pair, say $(i,a)$, while costs at other pairs are fixed, then for every $j\in\mathcal{S},b\in\mathcal{A}$, the function $\tilde{c}(i,a)\mapsto [\tilde{Q}^*]_{j,b}$ is a piece-wise linear function.
\end{remark}

Given any $c\in f^{-1}(\mathcal{V}_w)$, if an adversary falsifies the cost $c$ by injecting value $h$, i.e., $\tilde{c}=c+h$, the adversary can see how the falsification cause a change in $Q$-values. To be more specific, if $Q^*$ is the $Q$-values learned from cost $c$ by $Q$-learning algorithm (\ref{QLAl}), after the falsification $\tilde{c}$, the $Q$-value learned from $Q$-learning algorithm (\ref{FalQLe}) becomes $\tilde{Q}^*=Q^*+Gh$ if $\tilde{c}\in f^{-1}(\mathcal{V}_w)$. Then, an omniscient adversary can utilize (\ref{FrechetDerivative}) to find a way of falsification $h$ such that $\tilde{Q}^*$ can be driven to approach a desired set $\mathcal{V}_{w^\dag}$ bearing in mind that $D(w,w^\dag)=0$ for any two policies $w,w^\dag$. One difficulty is to see whether $\tilde{c}\in f^{-1}(\mathcal{V}_w)$ because the set $f^{-1}(\mathcal{V}_w)$ is now implicit. Thus, we resort to the following theorem.

\begin{theorem}\label{iffTheorem}
Let $\tilde{Q}^*\in\mathbb{R}^{S\times A}$ be the $Q$-values learned from the $Q$-learning algorithm (\ref{FalQLe}) with the falsified cost $\tilde{c}\in\mathbb{R}^{S\times A}$. Then $\tilde{Q}^* \in \mathcal{V}_{w^\dag}$ if and only if the falsified cost signals $\tilde{c}$ designed by the adversary satisfy the following conditions
\begin{equation}\label{FalCostConds}
\tilde{c}(i,a)> (\mathbf{1}_i - \beta P_{ia})^T(I-\beta P_{w^\dagger})^{-1}\tilde{c}_{w^\dagger}.
\end{equation}
for all $i\in \mathcal{S}$, $a\in \mathcal{A}\backslash \{w^\dag(i)\}$.
\end{theorem}
\begin{proof}[Sketch of Proof] 
If $\tilde{Q}^*\in\mathcal{V}_{w^\dag}$, from proof of Proposition \ref{Diffentialf}, we know $\tilde{Q}^*_{w^\dag}=(I-\beta P_{w^\dag})^{-1} \tilde{c}_{w^\dag}$ and the $i$th component of $\tilde{Q}^*_{w^\dag}$ is strictly less than $\tilde{Q}^*(i,a)$ for each $a\in\mathcal{A}\backslash \{w^\dag(i)\}$. That means
$\tilde{Q}^*(i,a)>\mathbf{1}_i^T \tilde{Q}^*_{w^\dag}$ which gives us $(\ref{FalCostConds})$. Conversely, if $\tilde{c}$ satisfy conditions (\ref{FalCostConds}), $\tilde{Q}^*\in\mathcal{V}_{w^\dag}$ due to the one-to-one, onto correspondence between $\tilde{c}$ and $\tilde{Q}^*$.
\end{proof}

With the results in Theorem \ref{iffTheorem}, we can characterize the set $f^{-1}(\mathcal{V}_w)$. Elements in $f^{-1}(\mathcal{V}_w)$ have to satisfy the conditions given in (\ref{FalCostConds}). Also, Theorem \ref{iffTheorem} indicates that if an adversary intends to mislead the agent to learn policy $w^\dag$, the falsified cost $\tilde{c}$ has to satisfy the conditions specified in (\ref{FalCostConds}). Note that for $a=w^\dag(i)$, $\tilde{c}(i,w^\dag(i))\equiv (\mathbf{1}_i - \beta P_{iw^\dag(i)})^T(I-\beta P_{w^\dagger})^{-1}\tilde{c}_{w^\dagger}$.

If the objective of an omnisicent attacker is to induce the agent to learn  policy $w^\dag$ while minimizing his own cost of attacking, i.e., the attack's problem we have formulated in (\ref{AttackPolicyObjctive}) in Section \ref{GenAttModel}. Given $\textrm{AttackCost}(\tilde{c})=\Vert \tilde{c}-c\Vert$ where $c$ is the true cost, the attacker's problem is to solve the following minimization problem
\begin{equation}\label{AttackMinPro}
\begin{aligned}
\min_{\tilde{c}\in\mathbb{R}^{S\times A}}\ \ \ &\Vert \tilde{c}-c \Vert\ \ \ 
s.t.\ (\ref{FalCostConds})
\end{aligned}
\end{equation}

\begin{remark}
If the norm in the attacker's problem (\ref{AttackMinPro}) is a Frobenius norm, the attacker's problem is a convex minimization problem which can be easily solved by omniscient attackers using software packages like MOSEK \cite{mosek2015mosek}, CVX \cite{grant2008cvx} etc. If $\textrm{AttackCost}(\tilde{c})$ is the number of state-action pair where the cost has been falsified, i.e.,, $\textrm{AttakCost}(\tilde{c})=\sum_i \sum_a \one_{\{c(i,a)\neq\tilde c(i,a)\}}$, then the attacker's problem becomes a combinatorial optimization problem \cite{wolsey2014integer}.
\end{remark}

\begin{remark}
If the actions available to an adversary only allow the adversary to falsify the true cost at certain states $\mathcal{S}'\subset \mathcal{S}$ (or/and at certain actions $\mathcal{A}'\subset\mathcal{A} $), then the adversary's problem (\ref{AttackMinPro}) becomes
$$
\begin{aligned}
\min_{\tilde{c}\in\mathbb{R}^{S\times A}}\ \ \ &\Vert \tilde{c}-c \Vert\\
s.t.\ \ \ \ \ &(\ref{FalCostConds})\\
&\tilde{c}(i,a)=c(i,a)\ \forall i\in \mathcal{S}\backslash \mathcal{S}',a\in\mathcal{A}\backslash\mathcal{A}'.  \\
\end{aligned}
$$
However, if an adversary can only falsify at certain states $\mathcal{S}'$, the adversary may not be able to manipulate the agent to learn $w^\dag$.
\end{remark}

Without loss of generality, suppose that the adversary can only falsify the cost at a subset of states $\mathcal{S}'=\{1,2,...,S'\}$. We rewrite the conditions given in (\ref{FalCostConds}) into a more compact form:
\begin{equation}\label{FalCostCondsCompact}
    \tilde{c}_a \geq (I-\beta P_a)(I-\beta P_{w^\dag})^{-1} \tilde{c}_{w^\dag}, \forall\ a\in\mathcal{A},
\end{equation}
where the equality only holds for one component of the vector, i.e., the $i$-th component satisfying $w(i)=a$. Partition the vector $\tilde{c}_a$ and $\tilde{c}_{w^\dag}$ in (\ref{FalCostCondsCompact}) into two parts, the part where the adversary can falsify the cost denoted by $\tilde{c}^{fal}_a,\tilde{c}^{fal}_{w^\dag}\in\mathbb{R}^{S'}$ and the part where the adversary cannot falsify $c_a^{true},c_{w^\dag}^{true}\in \mathbb{R}^{S-S'}$.

\begin{equation}\label{PartionedCostCon}
\begin{bmatrix}
\tilde{c}^{fal}_a\\ \hdashline[2pt/2pt]
c_a^{true}
\end{bmatrix}
 \geq \left[
    \begin{array}{c;{2pt/2pt}c}
        R_a  & Y_a \\ \hdashline[2pt/2pt]
        M_a & N_a
    \end{array}
\right]
\begin{bmatrix}
\tilde{c}^{fal}_{w^\dag}\\ \hdashline[2pt/2pt]
c_{w^\dag}^{true}
\end{bmatrix},\ \forall\ a\in\mathcal{A}
\end{equation}
where 
$$
\left[
    \begin{array}{c;{2pt/2pt}c}
        R_a  & Y_a \\ \hdashline[2pt/2pt]
        M_a & N_a
    \end{array}
\right]\coloneqq
(I-\beta P_a) (I-\beta P_{w^\dag})^{-1},\ \ \forall\ a\in\mathcal{A}
$$
and $R_a\in\mathbb{R}^{S'\times S'}, Y_a\in\mathbb{R}^{S'\times (S-S')},M_a\in\mathbb{R}^{(S-S')\times S'}, N_a\in\mathbb{R}^{(S-S')\times (S-S')}$.
Note that the $i$th component of $\tilde{c}_{w^\dag(i)}^{fal}$ is equal to the $i$ component of $\tilde{c}_{w^\dag}^{fal}$. If the adversary aims to mislead the agent to learn $w^\dag$, the adversary needs to design $\tilde{c}^{fal}_a,a\in\mathcal{A}$ such that the conditions in (\ref{PartionedCostCon}) hold. Whether the conditions in (\ref{PartionedCostCon}) are easy for an adversary to achieve or not depends on the true costs $c^{true}_{a},a\in\mathcal{A}$. The following results state that under some conditions on the transition probability, no matter what the true costs are, the adversary can find proper $\tilde{c}^{fal}_a,a\in\mathcal{A}$ such that conditions (\ref{PartionedCostCon}) are satisfied. For $i\in\mathcal{S}\backslash\mathcal{S}'$, if $w(i)=a$, we remove the rows of $M_a$ that correspond to the state $i\in\mathcal{S}\backslash\mathcal{S}'$. Denote the new matrix after the row removals by $\bar{M}_a$.

\begin{theorem}\label{PartialStatesAttacks}
Define $H\coloneqq [\bar{M}_{a_1}^T\ \bar{M}_{a_2}^T\ \cdots\ \bar{M}_{a_A}^T]^T \in \mathbb{R}^{(A(S-S')-S')\times S'}$. If there exists $x\in\mathbb{R}^{S'}$ such that $Hx<0$, i.e.,  the column space of $H$ intersects the negative orthant of $\mathbb{R}^{A(S-S')-S'}$,  then for any true cost, the adversary can find $\tilde{c}^{fal}_a,a\in\mathcal{A}$ such that conditions (\ref{PartionedCostCon}) hold.
\end{theorem}
\begin{proof}
We can rewrite (\ref{PartionedCostCon}) as $\tilde{c}_a^{fal} \geq R_a \tilde{c}_{w^\dag}^{fal}+ Y_a c^{true}_{w^\dag}$ and $c_a^{true}\geq M_a \tilde{c}^{fal}_{w^\dag}+N_a c^{true}_{w^\dag}$ for all $a\in\mathcal{A}$. If there exists $\tilde{c}_{w^\dag}^{fal}$ such that $M_a \tilde{c}^{fal}_{w^\dag}$ can be less than any given vector in $\mathbb{R}^{S-S'}$, then $c_a^{true}\geq M_a \tilde{c}^{fal}_{w^\dag}+N_a c^{true}_{w^\dag}$ can be satisfied no matter what the true cost is. We need $c_a^{true}\geq M_a \tilde{c}^{fal}_{w^\dag}+N_a c^{true}_{w^\dag}$ to hold for all $a\in\mathcal{A}$, which means that we need the range space of $[M_{a_1}^T,...,M_{a_A}^T]\in\mathbb{R}^{A(S-S')\times S'}$ to intersect the negative orthant.
By using the fact that $\tilde{c}(i,w^\dag(i))\equiv (\mathbf{1}_i - \beta P_{iw^\dag(i)})^T(I-\beta P_{w^\dagger})^{-1}\tilde{c}_{w^\dagger}$, we can give less stringent conditions. Actually, we only need the range space of  $H=[\bar{M}^T_{a_1},...,\bar{M}^T_{a_A}]\in\mathbb{R}^{(A(S-S')-S')\times S'}$ to intersection the negative orthant. If this is true, then these exists $\tilde{c}_{w^\dag}^{fal}$ such that $c_a^{true}\geq M_a \tilde{c}^{fal}_{w^\dag}+N_a c^{true}_{w^\dag}$ is feasible for all $a\in\mathcal{A}$.

As for conditions $\tilde{c}_a^{fal} \geq R_a \tilde{c}_{w^\dag}^{fal}+ Y_a c^{true}_{w^\dag}$, note that there are $S'\times A$ number of variables  $\tilde{c}_a^{fal},a\in\mathcal{A}$ and that $\tilde{c}_{w^\dag}^{fal}$ has been chosen such that conditions $c_a^{true}\geq M_a \tilde{c}^{fal}_{w^\dag}+N_a c^{true}_{w^\dag}$ are satisfied. One can  choose the remaining variables in $\tilde{c}_a^{fal},a\in\mathcal{A}$ sufficiently large to satisfy $c_a^{true}\geq M_a \tilde{c}^{fal}_{w^\dag}+N_a c^{true}_{w^\dag}$  due to the fact that $\tilde{c}(i,w^\dag(i))\equiv (\mathbf{1}_i - \beta P_{iw^\dag(i)})^T(I-\beta P_{w^\dagger})^{-1}\tilde{c}_{w^\dagger}$.
\end{proof}

Note that $H$ only depends on the transition probability and the discount factor, if an omniscient adversary can only falsify cost signals at states denoted by $\mathcal{S}'$, an adversary can check if the range space of $H$ intersects with the negative orthant of $\mathbb{R}^{A(S-S')}$ or not. If it does, the adversary can mislead the agent to learn $w^\dag$ by falsifying costs at a subset of state space no matter what the true cost is.

\begin{remark}
To check whether the condition on $H$ is true or not, one has to resort to Gordan's theorem \cite{broyden2001theorems}:
Either $Hx<0$ has a solution $x$, or $H^T y =0$ has a nonzero solution $y$ with $y\geq 0$. The adversary can use linear/convex programming software to check if this is the case. For example, by solving
\begin{equation}\label{GordanMin}
\begin{aligned}
\min_{y\in\mathbb{R}^{A(S-S')}}\ \ \ &\Vert H^T y \Vert\ \  s.t.\ \ \ \ \Vert y \Vert=1,\ y\geq 0,
\end{aligned}
\end{equation}
the adversary knows whether the condition about $H$ given in Theorem \ref{PartialStatesAttacks} is true or not. If the minimum of (\ref{GordanMin}) is $0$, the adversary cannot guarantee that, for any given true cost, the agent learns the policy $w^\dag$. If the minimum of (\ref{GordanMin}) is positive, there exists $x$ such that $Hx<0$. The adversary can select $\tilde{c}_{w^\dag}^{fal}=\lambda x$ and choose a sufficiently large $\lambda$ to make sure that conditions (\ref{PartionedCostCon}) hold, which means an adversary can make the agent learn the policy $w^\dag$ by falsifying costs at a subset of state space no matter what the true costs are. 
\end{remark}

\section{Numerical Example}\label{ExamNum}


In this section, we use the application of RL in water reservoir operations to illustrate the security issues of RL. Consider a RL agent aiming to create the best operation policies for the hydroelectric reservoir system described in Fig. \ref{fig:HydroelectricResvoirSystem}. The system consists of the following: (1) an inflow conduit regulated by $\textrm{Val}_0$, which can either be a river or a spillway from another dam; and (2) two spillways for outflow: the first penstock, $\textrm{Val}_1$, which is connected to the turbine and thus generates electricity, and the second penstock, $\textrm{Val}_2$, allowing direct water evacuation without electricity generation. We consider three reservoir levels: $\textrm{MinOperL}$, $\textrm{MedOperL}$, $\textrm{MaxExtL}$. Weather conditions and  the operation of valves are key factors that affect the reservoir level. In practice, there are usually interconnected hydroelectric reservoir systems located at different places which makes it difficult to find an optimal operational policy.

\begin{figure}\centering
\includegraphics[width=0.6\textwidth]{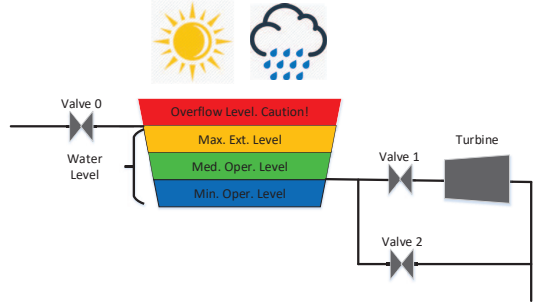}
\caption{A hydroelectric reservoir system.} \label{fig:HydroelectricResvoirSystem}
\end{figure}

For illustrative purposes, we only consider controlling of $\textrm{Val}_1$. Thus, we have two actions: either $a_1$,$\textrm{Val}_1=$ `shut down'; or $a_2$, $\textrm{Val}_1=$ `open'. Hence $\mathcal{A}=\{a_1,a_2\}$. We consider three states which represent three different reservoir levels, denoted by $\mathcal{S}=\{1,2,3\}$ where $1 (2,3)$ represents $\textrm{MaxExtL}$ ($\textrm{MedOperL}$, $\textrm{MinOperL}$,  respectively). The goal of the operators is to generate more electricity to increase economic benefits, which requires the reservoir  to store a sufficient amount of water to generate electricity. Meanwhile, the operator also aims to avoid possible overflows which can be caused by the unexpected heavy rain in the reservoir area or in upper areas. The operator needs to learn a safe policy, i.e., the valve needs to be open at state $1$ so that the cost at $c(1,a_1)$ needs to be high. We assume that the uncertain and intermittent nature is captured by the transition probability given by 
\begin{equation*}
    P_{a_1}=\begin{bmatrix}
    1 & 0 & 0\\
    0.6 & 0.4&0\\ 
    0.1 & 0.5 & 0.4\\
    \end{bmatrix}, P_{a_2}=\begin{bmatrix}
    0.3 & 0.7 & 0\\
    0.1& 0.2 &0.7\\
    0 & 0 & 1
    \end{bmatrix}.
\end{equation*}
And the true cost is assumed to be $c=[30\ -5; 6\ -10; 0\ 0 ]$. Negative cost can be interpreted as the reward for hydroelectric production. Let the discounted factor $\beta$ be $0.8$. The limit $Q$-values learned from $Q$-learning algorithm (\ref{QLAl}) is approximately $Q^*=\begin{bmatrix}
8.71 & -26.6129\\
-15.48 & -27.19\\
-19.12 & -15.30\\
\end{bmatrix}.$ The optimal policy thus is $w^*(1)=a_2, w^*(2)=a_2,w^*(3)=a_1$. Basically, the optimal policy indicates that one should keep the valve open to avoid overflowing and generate more electricity at $\textrm{MaxExtL}$. While at $\textrm{MinOperL}$, one should keep the valve closed to store more water for water supply and power generation purposes. From (\ref{LipCont}), we know that the resulting change in $Q^*$ under malicious falsification is bounded by the change in the cost with a Lipschitz constant ${1}/{(1-\beta)}$. To see this, we randomly generate $100$ falsifications $h\in\mathbb{R}^{3\times 2}$ using \textit{randi(10) *  rand(3,2)} in Matlab. For each falsified cost $\tilde{c}=c+h$, we obtain the corresponding $Q$-factors $\tilde{Q}^*$. We plot $\Vert \tilde{Q}^* - Q^* \Vert$ corresponding with $\Vert \tilde{c}-c \Vert$ for each falsification in Fig. \ref{fig:LipschitzTestFigure}. One can clearly see the bound given in (\ref{LipCont}). The result in Fig. \ref{fig:LipschitzTestFigure} corroborates Theorem \ref{LipContTheo}.

\begin{figure}\centering
\includegraphics[width=0.6\textwidth]{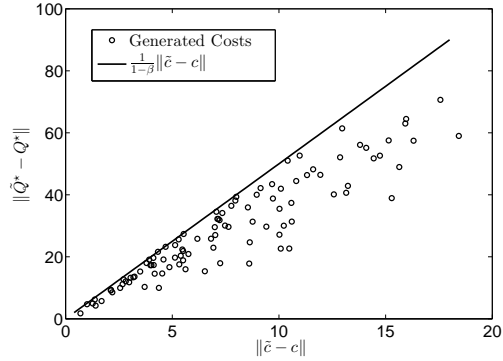}
\caption{$\Vert \tilde{Q}^* - Q^* \Vert$ versus $\Vert \tilde{c}-c \Vert$ with $100$ falsifications.} \label{fig:LipschitzTestFigure}
\end{figure}

Suppose that the adversary aims to mislead the agent to learn a policy $w^\dag$ where $w^\dag(1)=a_1$, $w^\dag(2)=a_2$, $w^\dag(3)=a_1$. The purpose is to keep the valve shut down at $\textrm{MaxExtL}$ which will cause overflow and hence devastating consequences. The adversary can utilize $D_{Q^*}(w^\dag)$ to see how much at least he has to falsify the original cost $\tilde{c}$ to achieve the desired policy $w^\dag$. The value of $D_{Q^*}(w^\dag)$ can be obtained by solving the following optimization problem:
$$
\begin{aligned}
\min_{Q\in\mathbb{R}^{3\times 2}}\ \ \ &\Vert Q - Q^* \Vert\\
s.t.\ \ \ & Q(1,a_1)\leq Q(1,a_2),  Q(2,a_2)\leq Q(1,a_1), Q(3,a_1)\leq Q(3,a_2).
\end{aligned}
$$
The value of $D_{Q^*}(w^\dag)$ is thus $17.66$. By (\ref{LipCont}), we know that to achieve $w^\dag$, the adversary has to falsify the cost such that $\Vert \tilde{c} -c  \Vert \geq (1-\beta)D_{Q^*}(w^\dag)=3.532$. If the actions available to the adversary are to perform only bounded falsification to one state-action pair with bound $3.5$, then it is impossible for the adversary to attain its goal, i.e., misleading the agent to the policy $w^\dag$ targeted by the adversary. Thus, in this MDP-FC, the robust region of $c$ to the adversary's desired policy $w^\dag$ is $3.532$.

\begin{figure}\centering
\includegraphics[width=0.90\textwidth]{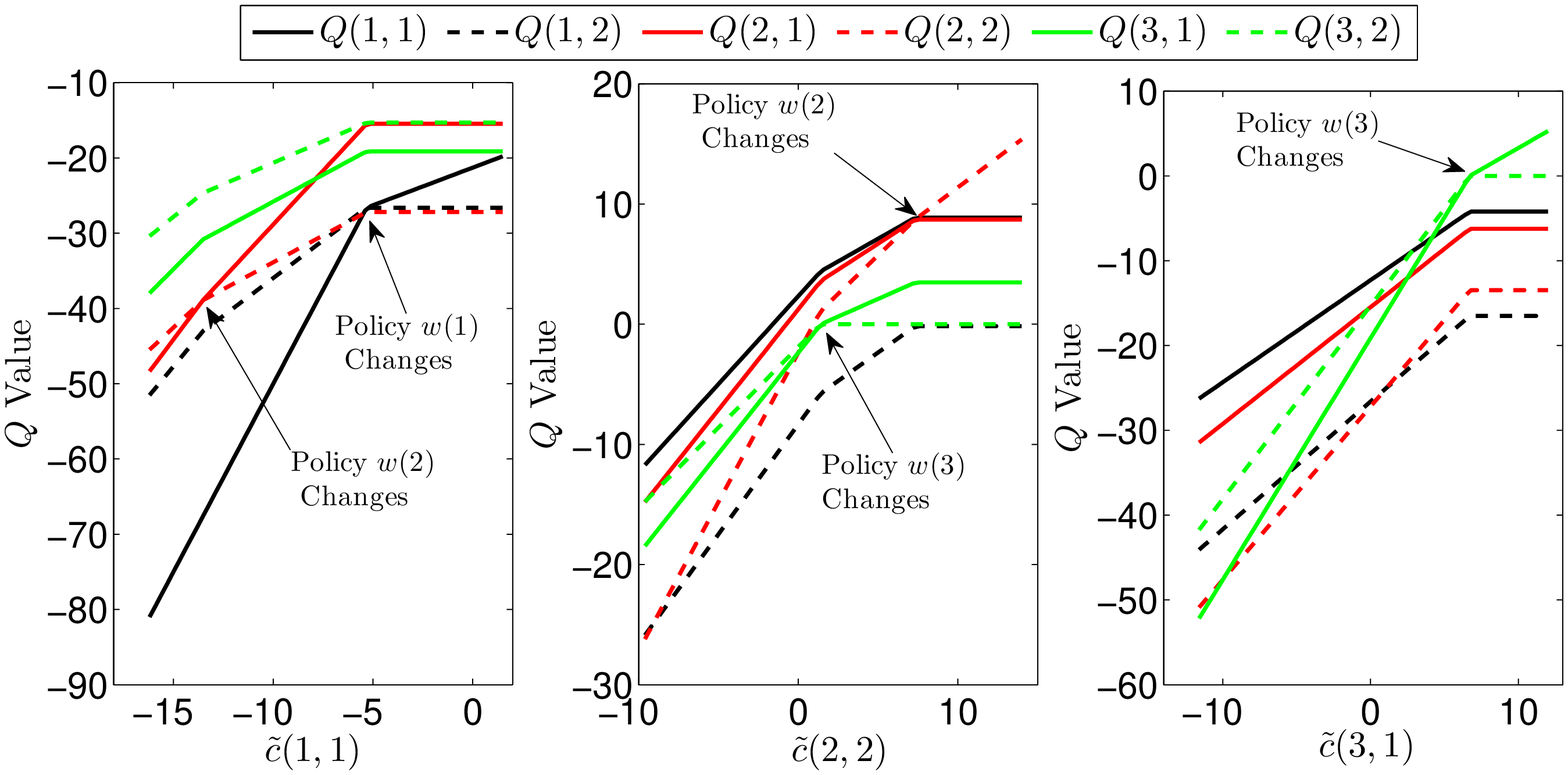}
\caption{The change of the limit $Q$-values when only the cost at one state-action pair is altered. Black line corresponds to state $1$, red line corresponds to state $2$ and green line corresponds to state $3$. Solid (dash) line corresponds to $a_1$ ($a_2$).} \label{fig:OneCostOneQStar}
\end{figure}

In Fig. \ref{fig:OneCostOneQStar}, we plot the change of the limit $Q$-values when only the cost at one state-action pair is falsified while the other components are fixed at $c=[9\ -5; 6\ -10; 0\ 0 ]$. We can see that when the other costs are fixed, for every $j\in\{1,2,3\},b\in\{a_1,a_2\}$ the function $\tilde{c}(i,a)\mapsto[\tilde{Q}^*]_{j,b}$ is piece-wise linear. And the change of slope happens only when the policy changes. This illustrates our argument about the differentiability of the mapping $\tilde{c}\mapsto \tilde{Q}^* $ in Proposition \ref{Diffentialf}. From the first two plots, one can see that changes in costs at one state can deviate the policy at another state. That is when altering the cost at $\textrm{MedOperL}$, an adversary can make the valve open at $\textrm{MinOperL}$ so that the reservoir cannot store enough water to maintain the water supply and generate electricity. When an adversary aims to manipulate the policy at one state, he does not have to alter the cost at this state. Fig. \ref{fig:TwoCostOneQStar1} illustrates Proposition \ref{Diffentialf} when costs corresponding to two state-action pairs are altered.

\begin{figure}\centering
\includegraphics[width=0.67\textwidth]{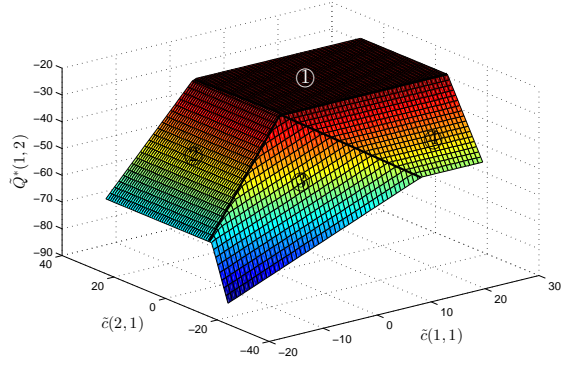}
\caption{the alteration of the limit $Q$-values when only the costs $\tilde{c}(2,1)$, $\tilde{c}(1,1)$ are altered.} \label{fig:TwoCostOneQStar1}
\end{figure}

Furthermore, to illustrate Proposition \ref{Diffentialf} in general cases, i.e., in $\mathbb{R}^{3\times 2}$, suppose $c=[9\ -5; 6\ -10; 0\ 0 ]$, the $Q$-factors learned from $c$ is $Q^*=[ -12.29\ -26.61;\ -15.47\  -27.19;\ -19.12\ -15.30]$. The optimal policy is thus $w^*(1)=a_2,w^*(2)=a_2,w^*(3)=a_1$. By (\ref{FrechetDerivative}) in Proposition \ref{Diffentialf}, the derivative of $f:\mathbb{R}^{3\times 2}$ at $c\in f^{-1}(\mathcal{V}_{w^*})$ is a linearly bounded map $G:\mathbb{R}^{3\times 2}\rightarrow\mathbb{R}^{3\times 2}$
\begin{equation}\label{ExampDeri}
[Gh]_{i,a}=0.8P_{ia}^T\Bigg(\begin{bmatrix}
1 & 0 & 0\\
0 & 1 & 0\\
0 & 0& 1
\end{bmatrix}-0.8\begin{bmatrix}
0.3 & 0.7 & 0\\
0.1 & 0.2 & 0.7\\
0.1 & 0.5 & 0.4\\
\end{bmatrix}\Bigg)^{-1}\begin{bmatrix}
h(1,a_2)\\
h(2,a_2)\\
h(3,a_1)
\end{bmatrix}+h(i,a).
\end{equation}
One can see that $G$ is a constant independent of $c$. Suppose that the adversary falsifies the cost from $c$ to $\tilde{c}$ by $h$, i.e., $\tilde{c}=c+h$ and $h=\begin{bmatrix}
0.6 & -0.2\\
1 & 2\\
0.4 & 0.7\\
\end{bmatrix}$. Then, $Gh=\begin{bmatrix}
3.74 & 3.92\\
4.70 & 5.68\\
4.39 & 4.21
\end{bmatrix}$ by (\ref{ExampDeri}). Thus, $\tilde{c}=c+h=\begin{bmatrix}
9.6 & -5.2\\
7 & -8\\
0.4 & 0.7\\
\end{bmatrix}$. The $Q$-factors learned from $\tilde{c}$ is $\tilde{Q}^*=\begin{bmatrix}
-8.55 & -22.69\\
-10.77 & -21.51\\
-14.73 & -11.08
\end{bmatrix}.$ The resulting policy is still $w^*$. One thus can see that $\tilde{Q}^*=Q^*+Gh$.

If an adversary aims to have the hydroelectric reservoir system operate based on a policy $w^\dag$, the falsified cost $\tilde{c}$ has to satisfy conditions given in (\ref{FalCostConds}). Let the targeted policy of the adversary be $w^\dag(1)=a_1,w^\dag(2)=a_2,w^\dag(3)=a_2$. If the adversary can deceptively falsify the cost at every state-action pair to any value, it is not difficult to find $\tilde{c}$ satisfying (\ref{FalCostConds}). For example, the adversary can first select $\tilde{c}_{w^\dag}=[\tilde{c}(1,a_1)\ \tilde{c}(2,a_2)\ \tilde{c}(3,a_2)]^T$, e.g., $\tilde{c}_{w^\dag}=[3\ 2\ 1]^T$. Then select cost at other state-action pairs following 
$\tilde{c}(i,a)= (\mathbf{1}_i - \beta P_{ia})^T(I-\beta P_{w^\dagger})^{-1}\tilde{c}_{w^\dagger}+\xi$ for $i\in\mathcal{S},a\in\mathcal{A}\backslash\{w^\dag(i)\}$, where $\xi>0$. Then, $\tilde{c}$ satisfies conditions (\ref{FalCostConds}). For example if an adversary choose $\xi=1$, the adversary will have $\tilde{c}=[3\ 10.86;  -1.34\ 2;  0.34\ 1]$. The $Q$-factors learned from $\tilde{c}$ is $\tilde{Q}^*=[15\ 18.46; 8.15\ 7.14; 5.99\ 5;]$. Thus, the resulted policy is the adversary desired policy $w^\dag$. Hence, we say if the adversary can deceptively falsify the cost at every state-action pair to any value, the adversary can make the RL agent learn any policy.

If an adversary can only deceptively falsify the cost at states $\mathcal{S}'$, we have to resort to Theorem \ref{PartialStatesAttacks} to see what he can achieve. Suppose that $\mathcal{S}'=\{1,2\}$ and the adversary desires policy $w^\dag(1)=a_1,w^\dag(2)=a_2,w^\dag(3)=a_2$. Given $\mathcal{S}'$ and $w^\dag$, (\ref{PartionedCostCon}) can be written as
\begin{equation}\label{ExamplePatialStates}
\begin{aligned}
&\begin{bmatrix}
\tilde{c}(1,a_1)\\
\tilde{c}(2,a_1)\\
c(3,a_1) 
\end{bmatrix}
\geq \begin{bmatrix}
    1.0000  &       0   &      0\\
   -2.0762  &  0.8095 &    2.2667\\
   -0.5905  &  -0.4762  &  2.0667
\end{bmatrix}\begin{bmatrix}
\tilde{c}(1,a_1)\\
\tilde{c}(2,a_2)\\
c(3,a_2)
\end{bmatrix},\\
&\begin{bmatrix}
\tilde{c}(1,a_2)\\
\tilde{c}(2,a_2)\\
c(3,a_2) 
\end{bmatrix}
\geq \begin{bmatrix}
    3.5333 &  -0.6667 &  -1.8667\\
   0 &   1.0000 &        0\\
         0 &    0    & 1.0000 \\
\end{bmatrix}\begin{bmatrix}
\tilde{c}(1,a_1)\\
\tilde{c}(2,a_2)\\
c(3,a_2)
\end{bmatrix}.\\
\end{aligned}
\end{equation}
Since the last row in the second equality is automatically satisfied, we have $H=[-0.5906 \ -0.4762]$ whose range space is $\mathbb{R}$ which intersects $(-\infty,0)$. Thus, no matter what values $c(3,a_1)$ and $c(3,a_2)$ are, the adversary can always find $\tilde{c}(1,a_1),\tilde{c}(2,a_2)$ such that $$
c(3,a_1)> M_{a_1}\begin{bmatrix}
\tilde{c}(1,a_1)\\
\tilde{c}(2,a_2)
\end{bmatrix}
+2.0667 \times c(3,a_2).
$$
Next, choose $\tilde{c}(2,a_1)$ and $\tilde{c}(1,a_2)$ by
$$
\begin{aligned}
&\tilde{c}(2,a_1)>\begin{bmatrix}
 -2.0762 &  0.8095 & 2.2667\\
\end{bmatrix}
\begin{bmatrix}
\tilde{c}(1,a_1)\\ 
\tilde{c}(2,a_2)\\ 
c(3,a_2) 
\end{bmatrix}\\
&\tilde{c}(1,a_2)>
\begin{bmatrix}
3.5333&-0.6667&-1.8667\\
\end{bmatrix}
\begin{bmatrix}
\tilde{c}(1,a_1)\\ 
\tilde{c}(2,a_2)\\ 
c(3,a_2) 
\end{bmatrix}.
\end{aligned}
$$
We hence can see that no matter what the true cost is, the adversary can make the RL agent learn $w^\dag$ by falsifying only the cost at sates $\mathcal{S}'=\{1,2\}$. It can also be easily seen that when the adversary can only falsify the cost at state $\mathcal{S}=\{1\}$, he can still make the RL agent learn the policy $w^\dag$ independent of the true cost.

\section{Conclusion and Future Work}
In this paper, a general framework has been introduced to study RL under deceptive falsifications of cost signals where a number of attack models have been presented. We have provided theoretical underpinnings for understanding the fundamental limits and performance bounds on the attack and the defense in RL systems. The robust region of the cost can be utilized by both offensive and defensive sides. A RL agent can leverage the robust region to evaluate the robustness to malicious falsifications. An adversary can use it to estimate whether certain objectives can be achieved or not. Conditions given in Theorem \ref{iffTheorem} provide a fundamental understanding of the possible strategic adversarial behavior of the adversary. Theorem \ref{PartialStatesAttacks} helps understand the attainability of an adversary's objective.
Future work would focus on investigating a particular attack model we have presented in Section \ref{GenAttModel} and developing defensive strategies based on the analytical tools we have introduced. 

%
%
\bibliographystyle{splncs04}
\bibliography{GameSec2019-SecurityofQlearning}

\end{document}